\documentclass[final,12pt]{alt2024} 

\newcount\Comments
\usepackage{bbm}
\usepackage{algorithm}
\usepackage{mathtools}
\usepackage{comment}

\PassOptionsToPackage{algo2e,ruled, linesnumbered}{algorithm2e}
\RequirePackage{algorithm2e}

\DeclarePairedDelimiter\ceil{\lceil}{\rceil}

\newcommand{\reals}{\mathbb{R}}
\newcommand{\naturals}{\mathbb{N}}

\usepackage{color}
\definecolor{darkgreen}{rgb}{0,0.5,0}
\definecolor{purple}{rgb}{1,0,1}
\newcommand{\kibitz}[2]{\ifnum\Comments=1\textcolor{#1}{#2}\fi}

\newcommand{\Acal}{\mathcal{A}}
\newcommand{\Bcal}{\mathcal{B}}
\newcommand{\Dcal}{\mathcal{D}}
\newcommand{\Fcal}{\mathcal{F}}

\newcommand{\Ecal}{\mathcal{E}}

\newcommand{\Lcal}{\mathcal{L}}

\newcommand{\Tcal}{\mathcal{T}}
\newcommand{\Vcal}{\mathcal{V}}
\newcommand{\Wcal}{\mathcal{W}}

\newcommand{\Xcal}{\mathcal{X}}
\newcommand{\Ycal}{\mathcal{Y}}

\DeclareMathOperator*{\argmin}{arg\,min}
\DeclareMathOperator*{\expect}{\operatorname{\mathbb{E}}}

\newcommand{\prob}{\mathbb{P}}
\newcommand{\trace}{\operatorname{tr}}
\newcommand{\indicator}{\mathbbm{1}}

\newcommand{\norm}[1]{\left\lVert#1\right\rVert}
\newcommand{\inner}[2]{\left\langle #1, #2 \right\rangle}

\title[Online Operator Learning]{Online Infinite-Dimensional Regression: Learning Linear Operators}
\usepackage{times}

\altauthor{
 \Name{Vinod Raman}\thanks{Equal Contribution}  \Email{vkraman@umich.edu}\\
 \Name{Unique Subedi}\footnotemark[1] \Email{subedi@umich.edu}\\
  \Name{Ambuj Tewari} \Email{tewaria@umich.edu}\\
 \addr University of Michigan
}

\begin{document}

\maketitle

\begin{abstract}%
We consider the problem of learning linear operators under squared loss between two infinite-dimensional Hilbert spaces in the online setting. We show that the class of linear operators with uniformly bounded $p$-Schatten norm is online learnable for any $p \in [1, \infty)$. On the other hand, we prove an impossibility result by showing that the class of uniformly bounded linear operators with respect to the operator norm is \textit{not} online learnable. Moreover, we show a separation between sequential uniform convergence and online learnability by identifying a class of bounded linear operators that is online learnable but uniform convergence does not hold. Finally, we prove that the impossibility result and the separation between uniform convergence and learnability also hold in the batch setting. 
\end{abstract}


\section{Introduction}
Learning operators between infinite-dimensional spaces is of fundamental importance in many scientific and engineering applications. For instance, the classical inverse problem is often modeled as learning an inverse mapping from a function space of observed data to the function space of underlying latent parameters, both of which are infinite-dimensional spaces \citep{kirsch2011introduction, tarantola2005inverse}. Such inverse problems have found widespread applicability in domains ranging from image processing, X-ray tomography, seismic inversion, and so forth \citep{neto2012introduction, uhlmann2003inside}. 
In addition, the solution to a partial differential equation is an operator from a space of functions specifying boundary conditions to the space of solution functions \citep{kovachki2021neural, li2020fourier}. Moreover, many of the traditional learning settings such as multi-task learning, matrix completion, and collaborative filtering can be modeled as learning operators between infinite-dimensional spaces \citep{abernethy2009new}.  Finally, many modern supervised learning applications involve working with datasets, where both the features and labels lie in high-dimensional spaces \citep{deng2009imagenet, santhanam2017generalized}. Thus, it is desirable to construct learning algorithms whose guarantees do not scale with the ambient dimensions of the problem.  

Most of the existing work in operator learning assumes some stochastic model for the data, which can be unrealistic in many applications. For instance, the majority of applications of operator learning are in the scientific domain where the data often comes from experiments \citep{lin2021operator}. Since experiments are costly, the data usually arrives sequentially and  with a strong temporal dependence that may not be adequately captured by a stochastic model. Additionally, given the high-dimensional nature of the data, one typically uses pre-processing techniques like PCA to project the data onto a low-dimensional space \citep{bhattacharya2021model, lanthaler2023operator}. Even if the original data has some stochastic nature, the preprocessing step introduces non-trivial dependencies in the observations that may be difficult to model. Accordingly, it is desirable to construct learning algorithms that can handle \textit{arbitrary} dependencies in the data. In fact, for continuous problems such as scalar-valued regression, one can often obtain guarantees similar to that of i.i.d. setting without making any assumptions on the data \citep{rakhlin2014online}.

In this paper, we study linear operator learning between two Hilbert spaces $\Vcal$ and $\Wcal$ in the \textit{adversarial online setting}, where one makes no assumptions on the data generating process \citep{cesa2006prediction}. In this model, a potentially adversarial nature plays a sequential game with the learner over $T$ rounds.  In each round $t \in [T]$, nature selects a pair of vectors $(x_t, y_t) \in \mathcal{\Vcal} \times \Wcal$ and reveals $x_t$ to the learner. The learner then makes a prediction $\hat{y}_t \in \Wcal.$ Finally, the adversary reveals the target $y_t$, and the learner suffers the loss $\norm{\hat{y}_t - y_t}^2_{\Wcal}$. A linear operator class $\mathcal{F} \subset \Wcal^{\Vcal}$ is online learnable if there exists an online learning algorithm such that for any sequence of labeled examples, the difference in cumulative loss between its predictions and the predictions of the best-fixed operator in $\mathcal{F}$ is small. In this work, we study the online learnability of linear operators and make the following contributions: 
\begin{itemize}
    \item[(1)] We show that the class of linear operators with uniformly bounded $p$-Schatten norm is online learnable with regret $O(T^{\, \max\left\{\frac{1}{2}, 1- \frac{1}{p} \right\}})$. We also provide a lower bound of $\Omega( T^{1-\frac{1}{p}})$, which matches the upperbound  for $p \geq 2$. 
    
    \item[(2)] We prove that the class of linear operators with uniformly bounded operator norm is not online learnable. Furthermore, we show that this impossibility result also holds in the batch setting.

    \item[(3)] Recently, there is a growing interest in understanding when uniform convergence and learnability are not equivalent \citep{montasser2019vc, hanneke2023multiclass}. Along this direction, we give a subset of bounded linear operators for which online learnability and uniform convergence are not equivalent. 
    
\end{itemize}

To make contribution (1), we upperbound the sequential Rademacher complexity of the loss class to show that sequential uniform convergence holds for the $p$-Schatten class for $p \in [1, \infty)$. For our hardness result stated in contribution (2), we construct a class with uniformly bounded operator norm that is not online learnable. 
Our construction in contribution (3) is inspired by and generalizes the example of \citet[Page 22]{natarajan1989some}, which shows a gap between uniform convergence and PAC learnability for multiclass classification. The argument showing that uniform convergence does not hold is a simple adaptation of the existing proof \citep{natarajan1989some}. However, since our loss is real-valued, showing that the class is learnable requires some novel algorithmic ideas, which can be of independent interest. 

\subsection{Related Works}
Regression between two infinite-dimensional function spaces is a classical statistical problem often studied in functional data analysis (FDA) \citep{wang2016functional, ferraty2006nonparametric}. In FDA, one typically considers $\Vcal$ and $\Wcal$ to be $L^2[0,1]$, the space of square-integrable functions, and the hypothesis class is usually a class of kernel integral operators.  We discuss the implication of our results to learning kernel integral operators in Section \ref{examples}. Recently, \cite{de2023convergence, nelsen2021random, mollenhauer2022learning} study learning more general classes of linear operators. However, all of these works are in the i.i.d. setting and assume a data-generating process.
Additionally, there is a line of work that uses deep neural networks to learn neural operators between function spaces \citep{kovachki2021neural, li2020fourier}. Unfortunately, there are no known learning guarantees for these neural operators. Closer to the spirit of our work is that of \citet{tabaghi2019learning}, who consider the agnostic PAC learnability of $p$-Schatten operators. They show that $p$-Schatten classes are agnostic PAC learnable. In this work, we complement their results by showing that $p$-Schatten classes are also {\em online} learnable. Going beyond the i.i.d. setting, there is a line of work that focuses on learning specific classes of operators from time series data \citep{brunton2016discovering, klus2020data}.

\section{Preliminaries}

\subsection{ Hilbert Space Basics}
Let $\mathcal{V}$ and $\mathcal{W}$ be real, separable, and infinite-dimensional Hilbert spaces. Recall that a Hilbert space is separable if it admits a countable orthonormal basis.  Throughout the paper, we let $\{e_n\}_{n=1}^{\infty}$ and $\{\psi_n\}_{n=1}^{\infty}$ denote a set of orthonormal basis for $\Vcal$ and $\Wcal$ respectively.  Then, any element $v \in \Vcal$ and $w \in \Wcal$ can be written as $v = \sum_{n=1}^{\infty} \beta_n e_n  \text{ and }  w = \sum_{n=1}^{\infty} \alpha_n \psi_n$
for  sequences $\{\beta_n\}_{n \in \naturals}$  and $ \{\alpha_n\}_{n=1}^{\infty}$ that are $\ell_2$ summable. 

Consider $w_1, w_2 \in \Wcal$ such that $w_1 = \sum_{n=1}^{\infty}\alpha_{n,1}\, \psi_n$ and $\sum_{n=1}^{\infty} \alpha_{n,2}\, \psi_n$. Then, the inner product between $w_1$ and $w_2$ is defined as $\inner{w_1}{w_2}_{\Wcal} := \sum_{n=1}^{\infty}\alpha_{n,1}\alpha_{n,2},$ and
it induces the norm $\norm{w_1}_{\Wcal}: = \sqrt{\inner{w_1}{w_1}}_{\Wcal} = \sqrt{\sum_{n=1}^{\infty} \alpha_{n,1}^2}$. One can equivalently define $\inner{\cdot}{\cdot}_{\Vcal}$ and $\norm{\cdot}_{\Vcal}$ to be the inner-product and the induced norm in the Hilbert space $\Vcal$. When the context is clear, we drop the subscript and simply write $\inner{\cdot}{\cdot}$ and $\norm{\cdot}$.

A linear operator $f: \Vcal \to \Wcal$ is a mapping that preserves the linear structure of the input. That is, $f(c_1 v_1 + c_2 v_2) = c_1 f(v_1) + c_2 f(v_2)$ for any $c_1, c_2 \in \reals$ and $v_1, v_2 \in \Vcal$. Let $\Lcal(\Vcal, \Wcal)$ denote the set of all linear operators from $\Vcal $ to $\Wcal$. A linear operator $f: \Vcal \to \Wcal$ is bounded if there exists a constant $c > 0$ such that $\norm{f(v)} \leq c \norm{v}$ for all $v \in \Vcal$. The quantity $\norm{f}_{\text{op}} := \inf \{c \geq 0 \, : \norm{f(v)} \leq c \norm{v}, \forall v \in \Vcal\} $ is called the operator norm of $f$. The operator norm induces the set of bounded linear operators, $\Bcal(\Vcal, \Wcal) = \{f \in \Lcal(\Vcal, \Wcal) \, \mid \norm{f}_{\text{op}} < \infty\},$
which is a Banach space with $\norm{\cdot}_{\text{op}}$ as the norm.

 For an operator $f \in \Lcal(\Vcal, \Wcal)$, let $f^{\star} : \Wcal \to \Vcal$ denote the adjoint of $f$. We can use $f$ and $f^{\star}$ to define a self-adjoint, non-negative operator $f^{\star}f: \Vcal \to \Vcal$. Moreover, the absolute value operator is defined as $|f| :=(f^{\star}f)^{\frac{1}{2}} $, which is the unique non-negative operator such that $|f| \circ |f| = f^{\star}f$. Given any operator $g: \Vcal \to \Vcal$, the trace of $g$ is defined as 
$\trace(g) = \sum_{n=1}^{\infty} \inner{g(e_n)}{e_n},$
where $\{e_n\}_{n=1}^{\infty}$ is any orthonormal basis of $ \Vcal$. The notion of trace and absolute value allows us to define the $p$-Schatten norm of $f$,
\[\norm{f}_p = \Big( \trace(|f|^{p})\Big)^{\frac{1}{p}},\]
 for all $p \in [1, \infty)$. Accordingly, we can define the $p$-Schatten class as 
 \[S_p(\Vcal, \Wcal ) = \{f \in \Lcal(\Vcal, \Wcal) \mid f \text{ is compact and }\norm{f}_p < \infty\}. \]

A linear operator $f:\Vcal \to \Wcal$ is compact if the closure of the set $\{f(v) \mid v \in \Vcal, \norm{v} \leq 1\}$ is compact. For a compact linear operator $f:\Vcal \to \Wcal$, there exists a sequence of orthonormal basis $\{\phi_{n}\}_{n=1}^{\infty} \subset \Vcal$ and $\{\varphi_{n}\}_{n=1}^{\infty} \subset \Wcal$ such that $f = \sum_{n=1}^{\infty} s_n(f)\,\, \varphi_n \otimes \phi_n$, where $s_n(f) \downarrow 0$ and $\varphi_n \otimes \phi_n$ denote the tensor product between $\varphi_n$ and $\phi_n$. This is the singular value decomposition of $f$ and the sequence $\{s_n(f)\}_{n=1}^{\infty}$ are the singular values of $f$. For $p \in [1, \infty)$, the $p$-Schatten norm of a compact operator is equal to the $\ell_p$ norm of the sequence $\{s_n(f)\}_{n \geq 1}$,
\begin{equation*}
   \norm{f}_p = \left(\sum_{n=1}^{\infty}s_n(f)^{p}\right)^{\frac{1}{p}}.
\end{equation*}
On the other hand, for a compact operator $f$, the $\ell_{\infty}$ norm of its singular values is equal to its operator norm, $\norm{f}_{\text{op}} = \norm{f}_{\infty} = \sup_{n \geq 1 }|s_n(f)|.$
Accordingly, for compact operators, the operator norm is referred to as $\infty$-Schatten norm, which induces the class
 \[S_{\infty}(\Vcal, \Wcal ) = \{f \in \Lcal(\Vcal, \Wcal) \mid f \text{ is compact and } \norm{f}_{\infty} < \infty\}. \]
Therefore, $S_{\infty}(\Vcal, \Wcal) \subset \Bcal(\Vcal, \Wcal)$. For a comprehensive treatment of the theory of Hilbert spaces and linear operators, we refer the reader to \cite{conway1990course} and \cite{weidmann2012linear}.

\subsection{Online Learning}\label{sec:ol}
Let $\Xcal \subseteq \Vcal$ denote the instance space, $\Ycal \subseteq \Wcal$ denote the target space, and $\mathcal{F} \subseteq \Lcal(\Vcal, \Wcal)$ denote the hypothesis class. In online linear operator learning, a potentially adversarial nature plays a sequential game with the learner over $T$ rounds. In each round $t \in [T]$, the nature selects a labeled instance $(x_t, y_t) \in \mathcal{X} \times \Ycal$ and reveals $x_t$ to the learner. The learner then uses all past examples $\{(x_i, y_i)\}_{i=1}^{t-1}$ and the newly revealed instance $x_t$ to make a prediction $\hat{y}_t \in \Ycal.$ Finally, the adversary reveals the target $y_t$, and the learner suffers the loss $\norm{\hat{y}_t - y_t}^2_{\Wcal}$.  Given   $\mathcal{F}$, the goal of the learner is to make predictions such that its regret, defined as a difference between the cumulative loss of the learner and the best possible cumulative loss over operators in $ \mathcal{F}$, is small. 

\begin{definition}[Online Linear Operator Learnability]\label{OOL}
A linear operator class $\Fcal \subseteq \Lcal(\Vcal, \Wcal)$ is online learnable if there exists an algorithm $\mathcal{A}$ such that its expected regret is
$$\emph{\texttt{R}}_{\Acal}(T, \Fcal) := \sup_{(x_1, y_1), \ldots, (x_T, y_T)}\, \mathbb{E}\left[\sum_{t=1}^T \norm{\Acal(x_t) - y_t}^2 - \inf_{f \in \mathcal{F}}\sum_{t=1}^T  \norm{f(x_t) - y_t}^2\right]  $$
is a non-decreasing, sublinear function of $T$.
\end{definition}

Unlike when $\Vcal$ is finite-dimensional, the class $\Fcal = \Lcal(\Vcal, \Wcal)$ is not online learnable when $\Vcal$ is infinite-dimensional (see Section \ref{sec:lowerbounds}). Accordingly, we are interested in understanding for which subsets $\Fcal \subset \Lcal(\Vcal, \Wcal)$ is online learning possible. Beyond online learnability, we are also interested in understanding when a probabilistic property called the sequential uniform convergence holds for the loss class $\{(x, y) \mapsto \norm{f(x)-y}^2 \, : \, f \in \Fcal\}.$

\begin{definition}[Sequential Uniform Convergence]\label{ouniform}
Let $\{(X_t, Y_t)\}_{t=1}^T$ be an arbitrary sequence of random variables defined over an appropriate probability space on $\Xcal \times \Ycal$, and $ \mathcal{C} = \{\mathcal{C}_t\}_{t=0}^{T-1}$  be an arbitrary filtration such that $(X_t, Y_t)$ is $\mathcal{C}_t$-measurable. Given a linear operator class $\Fcal \subseteq \Lcal(\Vcal, \Wcal)$, we say that sequential uniform convergence holds for a loss class $\{(x, y) \mapsto \norm{f(x)-y}^2 \, : \, f \in \Fcal\}$ if
\[\limsup_{T \to \infty}\,\, \sup_{\bf{P}}\,\,   \mathbb{E}\left[ \sup_{f \in \Fcal} \left|\frac{1}{T} \sum_{t=1}^T \left( \norm{f(X_t)-Y_t}^2 - \mathbb{E}[\norm{f(X_t)-Y_t}^2 \mid \mathcal{C}_{t-1}] \right)\right| \right] =0. \]
Here, the supremum is taken over all joint distributions $\bf{P}$ of $\{(X_t, Y_t)\}_{t=1}^T$.
\end{definition}

A general complexity measure called the sequential Rademacher complexity characterizes sequential uniform convergence \citep{rakhlin2015online, rakhlin2015sequential}. 

\begin{definition}[Sequential Rademacher Complexity]\label{sRad}
Let $\sigma = \{\sigma_i\}_{i=1}^T$ be a sequence of independent Rademacher random variables and $(x,y) = \{(x_t, y_t)\}_{t=1}^T$ be a sequence of functions $(x_t,y_t) : \{-1,1\}^{t-1} \to \Xcal \times \Ycal$. Then, the sequential Rademacher complexity of the loss class $ \{(v, w) \mapsto \norm{f(v)-w}^2 \, : \, f \in \Fcal\}$ is defined as 
\[\emph{\text{Rad}}_{T}( \Fcal) = \sup_{x,y} \,\expect\left[ \sup_{f \in \Fcal} \,\sum_{t=1}^T \sigma_t \norm{f(x_t(\sigma_{< t}))- y_t(\sigma_{<t})}^2\right],\]
where $\sigma_{< t} = (\sigma_1, \ldots, \sigma_{t-1}).$
\end{definition}
\noindent If there exists a $B > 0$ such that $\sup_{f, v, w} \norm{f(v)-w}^2 \leq B$, then Theorem 1 of \cite{rakhlin2015sequential} implies that the sequential uniform convergence holds for the loss class $ \{(v, w) \mapsto \norm{f(v)-w}^2 \, : \, f \in \Fcal\}$ if and only if $\text{Rad}_{T}(\Fcal) = o(T)$. Given this equivalence, in this work, we only rely on the sequential Rademacher complexity of $\Fcal$ to study its sequential uniform convergence property. 

\section{$p$-Schatten Operators are Online Learnable}\label{sec:pSchatten}
In this section, we show that every uniformly bounded subset of  $S_p(\Vcal, \Wcal)$ is online learnable. Despite not making any distributional assumptions, the rates in Theorem \ref{Sp} match the lowerbounds in the batch settig established in Section \ref{sec:PAC}. This complements the results by \cite{rakhlin2014online}, who show that the rates for scalar-valued regression with squared loss are similar for online and PAC learning. 


\begin{theorem}[Uniformly Bounded Subsets of $S_p(\Vcal, \Wcal)$ are Online Learnable]\label{Sp}
Fix $c> 0$. Let $\Xcal = \{v \in \Vcal \mid \norm{v} \leq 1\}$ denote the instance space, $\Ycal = \{w \in \Wcal \mid \norm{w} \leq c\}$ denote the target space, and  $\Fcal_p = \{f \in S_{p}(\Vcal, \Wcal) \, \mid \, \norm{f}_{p} \leq c\}$ be the hypothesis class for $p \in [1,\infty]$. Then, 
\[\inf_{\Acal} \, \emph{\texttt{R}}_{\Acal}(T, \Fcal_p) \leq\, \,2\,\emph{\text{Rad}}_T(\Fcal_p) \,\,\leq \,\, 6c^2\,T^{\,\max\left\{\frac{1}{2}, 1-\frac{1}{p} \right\}}. \] 
\end{theorem}
\noindent Theorem \ref{Sp} implies the regret $O(\sqrt{T})$ for $p \in [1,2]$ and the regret  $O(T^{1- \frac{1}{p}})$ for $p > 2$. When $p = \infty$, the regret bound implied by Theorem \ref{Sp} is vacuous. Indeed, in Section \ref{sec:lowerbounds}, we prove that any uniformly bounded subset of $S_{\infty}(\Vcal, \Wcal)$  is not online learnable.


Our proof of Theorem \ref{Sp} relies on Lemma \ref{radsum}
which shows that the $q$-Schatten norm of Rademacher sums of rank-1 operators concentrates for every $q \geq 1$. The proof of Lemma \ref{radsum} is in Appendix \ref{appdx:tech_lem}. 
  \begin{lemma}[Rademacher Sums of Rank-1 Operators]\label{radsum}
    Let $\sigma = \{\sigma_i\}_{i=1}^T$ be a sequence of independent Rademacher random variables and $ \{(v_t, w_t)\}_{t=1}^T$ be any sequence of functions $(v_t,w_t) : \{-1,1\}^{t-1} \to \{v \in \Vcal \, : \norm{v} \leq c_1\}\times \{w \in \Wcal \, : \norm{w} \leq c_2\}$. Then, for any $q \geq 1$, we have
\[\expect \left[\norm{\sum_{t=1}^{T}\sigma_t\,  v_t(\sigma_{< t}) \otimes w_t(\sigma_{< t})}_q \right] \leq c_1\, c_2\, T^{\max\left\{\frac{1}{2}, \frac{1}{q}\right\}}\]
  \end{lemma}

\noindent  Lemma \ref{radsum} extends Lemma 1 in \citep{tabaghi2019learning} to the non-i.i.d. setting. In particular, the rank-1 operator indexed by $t$ can depend on the Rademacher subsequence $\sigma_{<t}$, whereas they only consider the case when the rank-1 operators are independent of the Rademacher sequence. In addition, \cite{tabaghi2019learning} use a non-trivial result from convex analysis, namely the fact that  $A \mapsto \trace(h(F))$ is a convex functional on the set $\{F \in \Tcal \mid \text{spectra}(F) \subseteq [\alpha, \beta]\}$ for any convex function $h$ and the class of finite-rank self-adjoint operators $\Tcal$. Our proof of Lemma \ref{radsum}, on the other hand, only uses standard inequalities.

Equipped with Lemma \ref{radsum}, our proof of Theorem \ref{Sp} follows by upper bounding the sequential Rademacher complexity of the loss class. Although this proof of online learnability is non-constructive, 
we can use Proposition 1 from \citep{rakhlin2012relax} to design an explicit online learner that achieves the matching regret given access to an oracle that computes the sequential Rademacher complexity of the class.  Moreover, online mirror descent (OMD) with the $\norm{f}_p^p$ regularizer also achieves the rates established in Theorem \ref{Sp}. In particular, OMD with the strongly convex regularizer $\norm{f}_2^2$ guarantees regret $O(\sqrt{T})$ for $p =2$. The $O(\sqrt{T})$ regret bound for $\mathcal{F}_2$ immediately implies an $O(\sqrt{T})$ regret bound for all $\mathcal{F}_p \subseteq \mathcal{F}_2$ in $p \in [1,2]$ by monotonicity. For $p>2$, the Clarkson-McCarthy inequality \citep{bhatia1988clarkson} implies that $\norm{f}_p^p$  is $p$-uniformly convex and thus OMD with this regularizer obtains the regret of $O(T^{1-\frac{1}{p}})$ \citep{sridharan2010convex, srebro2011universality}. That said, Theorem \ref{Sp} establishes a stronger guarantee-- not only are these classes online learnable but they also enjoy sequential uniform convergence.

\subsection{Examples of $p$-Schatten class}\label{examples}
In this section, we provide examples of operator classes with uniformly bounded $p$-Schatten norm.\\

\noindent \textbf{Uniformly bounded operators w.r.t. $\norm{\cdot}_{\text{op}}$ when either $\Vcal$ or $\Wcal$ is finite-dimensional.}
If either the input space $\Vcal$ or the output space $\Wcal$ is finite-dimensional, then the class of bounded linear operators $\Bcal(\Vcal, \Wcal)$ is $p$-Schatten class for every $p \in [1, \infty]$. This is immediate because for every $f \in \Bcal(\Vcal, \Wcal)$, either the operator $f^{\star}f: \Vcal \to \Vcal$ or $ff^{\star}: \Wcal \to \Wcal$ is a bounded operator that maps between two finite-dimensional spaces. Let $\norm{f}_{\text{op}} \leq c$ and $\min\{\text{dim}(\Vcal), \text{dim}(\Wcal)\} = d < \infty$. Since $ f^{\star}f$ and $ ff^{\star}$ have the same singular values and one of them has rank at most $d$, both of them must have rank at most $d$.  Let $s_1 \geq s_2 \ldots \geq s_d \geq 0$ denote all  singular values of $f^{\star}f$. Then,
$\norm{f}_p = \left(\sum_{i=1}^d s_i^p \right)^{\frac{1}{p}} \leq c \, d^{\frac{1}{p}} <\infty,$
where we use the fact that $s_i \leq c$ for all $i$.
Since $\norm{f}_2 \leq c\,\sqrt{d}$, Theorem \ref{Sp} implies that  $\Fcal= \{f \in \Bcal(\Vcal, \Wcal) \, \mid \norm{f}_{\text{op}} \leq c\}$ is online learnable with regret at most $6c^2d \sqrt{T}$.\\

\noindent \textbf{Kernel Integral Operators.} 
Let $\Vcal$ denote a Hilbert space of functions defined on some domain $\Omega$. Then, a kernel $K : \Omega \times \Omega \to \reals$ defines an integral operator $f_K : \Vcal \to \Wcal$ such that  $f_K(v(r)) = \int_{\Omega}\, K(r,s)\,  v(s) \, d\mu(s),$
for some measure space $(\Omega, \mu)$. 
Now define a class of integral operators,
\[\Fcal = \left\{f_K \, : \,  \int_{\Omega} \int_{\Omega} \, |K(r,s)|^2 \, d\mu(r) \,  d\mu(s) \leq c^2\right\},\]
induced by all the kernels whose $L^2$ norm is bounded by $c$. It is well known that $\norm{f}_2 \leq c$ for every $f \in \Fcal$ (see \citep[Page 267]{conway1990course} and \citep[Theorem 6.11]{weidmann2012linear}) . Thus, Theorem \ref{Sp} implies that $\Fcal$ is online learnable with regret  $6c^2 \sqrt{T}$.

\section{Lower Bounds and Hardness Results}\label{sec:lowerbounds}
In this section, we establish lower bounds for learning uniformly bounded subsets of $S_p(\Vcal, \Wcal)$ for $p \in [1, \infty]$.  

\begin{theorem}[Lower Bounds for Uniformly Bounded Subsets of $S_p(\Vcal, \Wcal)$]\label{Sp_lower}
Fix $c>0$. Let $\Xcal = \{v \in \Vcal \mid \norm{v} \leq 1\}$ denote the instance space, $\Ycal = \{w \in \Wcal \mid \norm{w} \leq c\}$ denote the target space, and  $\Fcal_p = \{f \in S_{p}(\Vcal, \Wcal) \, \mid \, \norm{f}_{p} \leq c\}$  be the hypothesis class for $p \in [1, \infty]$. Then, we have
\[\inf_{\Acal}\, \emph{\texttt{R}}_{\Acal}(T, \Fcal_p)  \geq  c^2\, T^{1-\frac{1}{p}}. \]
\end{theorem}

\noindent Theorem \ref{Sp_lower} shows a linear lowerbound of  $c^2\,T$ for $p=\infty$, thus implying that the class $\Fcal_{\infty}$ \emph{is not online learnable}. For $p \in [2, \infty)$, the lowerbound in Theorem \ref{Sp_lower} matches the upperbound in Theorem \ref{Sp} up to a factor of $6$.  However, in the range $p \in [1,2)$, our upperbound saturates at the rate $\sqrt{T}$, while the lower bound gets progressively worse as $p$ decreases. It remains an open problem to find the optimal regret of learning $\Fcal_p$ for $p \in [1,2)$.

\begin{proof}(of Theorem \ref{Sp_lower})
Fix an algorithm $\Acal$, and consider a labeled stream $\{(e_t, c\,\sigma_t\psi_t)\}_{t=1}^T$ where $\sigma_t \sim \text{Unif}(\{-1,1\})$. Then, the expected loss of  $\Acal$ is
\begin{equation*}
    \begin{split}
        \expect \left[\sum_{t=1}^T \norm{\Acal(e_t) - c\,\sigma_t \psi_t}^2 \right] &\geq \sum_{t=1}^T \left(\expect \left[ \norm{\Acal(e_t) - c\, \sigma_t \psi_t}\right] \right)^2\\
        &= \sum_{t=1}^T \left( \expect_{\Acal}\left[\frac{1}{2} \norm{\Acal(x_t)- c\, \psi_t} + \frac{1}{2} \norm{\Acal(x_t) + c\, \psi_t}\right]\right)^2\\
        &\geq \sum_{t=1}^T \left( \frac{1}{2} \norm{c\,\psi_t - (-c\, \psi_t)}\right)^2 = \sum_{t=1}^T c^2\, \norm{\psi_t}^2 = c^2\, T.
    \end{split}
\end{equation*} 
The first inequality above is due to Jensen's, whereas the second inequality is the triangle inequality. 

To establish the upper bound on the optimal cumulative loss amongst operators in $\Fcal_p$, consider the operator $f_{\sigma,p} := \sum_{t=1}^T \frac{c\, \sigma_t}{T^{1/p}} \, \, \psi_t \otimes e_t$.  As the singular values of $f_{\sigma,p}$ are $\{c\, \sigma_t T^{-1/p}\}_{t=1}^T$, we have
\[\norm{f_{\sigma,p}}_p = \left(\sum_{t=1}^T \left|\frac{c\, \sigma_t}{T^{1/p}} \right|^p \right)^{1/p} = \left(\sum_{t=1}^T \frac{c^p}{T} \right)^{1/p} 
 = c \quad \text{ for } p \in [1, \infty).\]
Similarly, $\norm{f_{\sigma, \infty}}_{\infty} = \norm{\sum_{t=1}^{T} c \sigma_t \psi_t \otimes e_t}_{\infty}= \max_{t\geq 1 } |c\, \sigma_t| =c$. That is, $f_{\sigma,p } \in \Fcal_p$ for all $p \geq 1$. Thus, we obtain that
\begin{equation*}
    \begin{split}
        \expect \left[\inf_{f \in \Fcal_p} \sum_{t=1}^T  \norm{f(e_t) - c\sigma_t \psi_t}^2\right] \leq       \expect \left[ \sum_{t=1}^T  \norm{f_{\sigma,p}(e_t) - c\sigma_t \psi_t}^2\right]
        &= \expect \left[ \sum_{t=1}^T \norm{\frac{c\,\sigma_t}{T^{1/p}}  \psi_t - c\sigma_t \psi_t}^2 \right]\\
        &=   \sum_{t=1}^T c^2\left(1-\frac{1}{T^{1/p}} \right)^2\\
        &\leq \sum_{t=1}^T c^2\left(1-\frac{1}{T^{1/p}} \right)  = c^2\,T - c^2\,T^{1-\frac{1}{p}}.\\
    \end{split}
\end{equation*}
Therefore, we have shown that the regret of $\Acal$ is 
\begin{equation*}
    \begin{split}
        \expect \Bigg[\sum_{t=1}^T \norm{\Acal(e_t) - c\,\sigma_t \psi_t}^2  &- \inf_{f \in \Fcal_p} \sum_{t=1}^T  \norm{f(e_t) - c\,\sigma_t \psi_t}^2 \Bigg] \geq c^2\,T^{1-\frac{1}{p}}.
    \end{split}
\end{equation*}
Our proof uses a random adversary, and the expectation above is taken with respect to both the randomness of the algorithm and the stream. However, one can use the probabilistic method to argue that for every algorithm, there exists a fixed stream forcing the claimed lowerbound. This completes our proof.
\end{proof}

\subsection{Lower Bounds in the Batch Setting}\label{sec:PAC}
In the batch setting, the learner is provided with $n \in \mathbb{N}$ i.i.d. samples $S=\{(x_i, y_i)\}_{i=1}^n$ from  a joint distribution $\Dcal$ on $\mathcal{X} \times \mathcal{Y}$ that is unknown to the learner. Using the sample $S$, the learner then finds a predictor $\hat{f}_n \in \mathcal{Y}^{\mathcal{X}} $ using some learning rule. We will abuse notation and use $\hat{f}_n$ to denote both the learning rule and the predictor returned by it. Given a linear operator class $\Fcal \subseteq \Lcal(\Vcal, \Wcal)$, the goal of the learner is to find an estimator $\hat{f}_n$ with a small worst-case expected excess risk
\[\mathcal{E}_n(\mathcal{F},  \hat{f}_n) := \sup_{\Dcal}\, \expect_{S_n \sim \Dcal^n} \left[\expect_{(x,y) \sim \Dcal}\left[\norm{\hat{f}_n(x)-y}^2\right] - \inf_{f \in \mathcal{F}} \expect_{(x,y) \sim \Dcal}\left[\norm{f(x)-y}^2\right] \right]. \]
The minimax excess risk for learning the function class $\mathcal{F}$ is then defined as $\mathcal{E}_n(\mathcal{F}) = \inf_{\hat{f}_n}  \mathcal{E}(\mathcal{F}, \hat{f}),$
where the infimum is over all possible learning rules. 
We adopt the minimax perspective to define agnostic batch learnability. 
\begin{definition}[Batch Learnability]
 A linear operator class $\mathcal{F} \subseteq \Lcal(\Vcal, \Wcal)$ is batch learnable  if and only if  $\limsup_{n \to \infty} \, \mathcal{E}_n(\mathcal{F})= 0. $
\end{definition}

Our results in Section \ref{sec:pSchatten} immediately provide an upperbound on $\Ecal_n(\Fcal)$ because $\Ecal_n(\Fcal)$ is upper bounded by the batch Rademacher complexity of $\Fcal$, which is further upper bounded by its sequential analog. Similar upperbounds on batch Rademacher complexity of $\Fcal$ were also provided by \cite{tabaghi2019learning}. In this section, we complement these results by providing lower bounds on $\Ecal_n(\Fcal)$. 

\begin{theorem}[Batch Lower Bounds for Uniformly Bounded Subsets of $S_p(\Vcal, \Wcal)$]\label{Sp_lower_PAC}
Fix $c>0$. Let $\Xcal = \{v \in \Vcal \mid \norm{v} \leq 1\}$ denote the instance space, $\Ycal = \{w \in \Wcal \mid \norm{w} \leq c\}$ denote the target space, and  $\Fcal_p = \{f \in S_{p}(\Vcal, \Wcal) \, \mid \, \norm{f}_{p} \leq c\}$  be the hypothesis class for $p \in [1, \infty]$. Then, we have
\[\Ecal_n(\Fcal) \geq \frac{c^2}{12} \max \left\{  n^{-\frac{1}{p-1}}, 
  n^{-\frac{2}{p}} \right\}.\]
\end{theorem}
 Theorem \ref{Sp_lower_PAC} shows a non-vanishing lowerbound of  $\frac{c^2}{12}$ for $p=\infty$, immediately implying that the class $\Fcal_{\infty}$ \emph{is not batch learnable}. For $p \in [2, \infty)$, \cite{tabaghi2019learning} provides an upperbound of $O(n^{-\frac{1}{p}})$, whereas our lowerbound is $\Omega(n^{-\frac{1}{p-1}})$. Additionally, for $p\in [1,2)$, there is also a gap between our lowerbound of $\Omega(n^{-\frac{2}{p}})$ and \cite{tabaghi2019learning}'s upperbound of $O(n^{-\frac{1}{2}})$. Thus, it remains to find the optimal rates for learning $\Fcal_p$ for every $p \in [1,\infty)$.

\section{Online Learnability without Sequential Uniform Convergence}\label{sec:sep}
In learning theory,  the uniform law of large numbers is intimately related to the learnability of a hypothesis class. For instance, a binary hypothesis class is PAC learnable if and only if the hypothesis class satisfies the i.i.d. uniform law of large numbers \citep{ShwartzDavid}.  An online equivalent of this result states that a binary hypothesis class is \textit{online} learnable if and only if the hypothesis class satisfies the sequential uniform law of large numbers \citep{rakhlin2015sequential}.  However, in a recent work, \cite{hanneke2023multiclass} show that uniform convergence and learnability are not equivalent for online multiclass classification. A key factor in \cite{hanneke2023multiclass}'s proof is the unboundedness of the size of the label space. This unboundedness is critical as the equivalence between uniform convergence and learnability continues to hold for multiclass classification with a finite number of labels \citep{DanielyERMprinciple}. Nevertheless, the number of labels alone cannot imply a separation. This is true because a real-valued function class (say $\mathcal{G} \subseteq [-1,1]^{\Xcal}$ where the size of label space is uncountably infinite) is online learnable with respect to absolute/squared-loss if and only if the uniform convergence holds \citep{rakhlin2015online}. In this section, we show an analogous separation between uniform convergence and learnability for online linear operator learning. As the unbounded label space was to \citet{hanneke2023multiclass},  the infinite-dimensional nature of the target space
is critical to our construction exhibiting this separation. Mathematically, a unifying property of \cite{hanneke2023multiclass}'s and our construction is the fact that the target space $\Ycal$ is not \textit{totally bounded} with respect to the pseudometric defined by the loss function. 


The following result establishes a separation between uniform convergence and online learnability for bounded linear operators. In particular, we show that there exists a class of bounded linear operators $\Fcal$ such that the sequential uniform law of large numbers does not hold, but $\Fcal$ is online learnable. 

\begin{theorem}[Sequential Uniform Convergence $\not \equiv$ Online Learnability]\label{UCneqLearn}
Let $\Xcal = \{v \in \Vcal \mid \sum_{n=1}^{\infty} |c_n| \leq 1 \text{ where } v = \sum_{n=1}^{\infty} c_n e_n\}$ be the instance space and $\Ycal = \{v \in \Vcal \mid \norm{v} \leq 1\}$ be the target space. Then, there exists a function class $\Fcal \subset S_{1}(\Vcal, \Vcal)$ such that the following holds:
\begin{itemize}
    \item[\emph{(i)}] $\emph{\text{Rad}}_{T}(\Fcal) \geq \frac{T}{2}$
    \item[\emph{(ii)}]  $\inf_{\Acal} \emph{\texttt{R}}_{\Acal}(T, \Fcal) \leq 2 + 8 \sqrt{T\log{(2T)}}$.
\end{itemize}
\end{theorem}

\begin{proof}
    For a natural number $k \in \naturals$, define an operator $f_k : \Vcal \to \Vcal$ as 
    \begin{equation}\label{natfunc}
        f_k := \sum_{n=1}^{\infty} b_k[n]\, \,  e_k \otimes e_n = e_k \otimes \sum_{n=1}^{\infty} b_k[n] \, e_n
    \end{equation}
    where $b_k$ is the binary representation of the natural number $k$ and $b_k[n]$ is its $n^{th}$ bit. Define $\Fcal = \{f_k \mid k \in \naturals\} \cup \{f_0\}$ where $f_0 = 0$ . 

    We begin by showing that $\Fcal \subset S_{1}(\Vcal, \Vcal)$. For any $\alpha, \beta \in \reals$ and $v_1, v_2 \in \Vcal$, we have 
    \[f_k(\alpha v_1 + \beta v_2 ) =\sum_{n=1}^{\infty} b_k[n]\, \,  \inner{e_n}{\alpha v_1+ \beta v_2} e_k = \alpha f_k(v_1) + \beta f_k(v_2).\]
    Thus, $f_k$ is a linear operator. Note that $f_k$ is defined in terms of singular value decomposition, and has only one non-zero singular value along the direction of $e_k$. Therefore, 
    \[\norm{f_k}_1 = \sum_{n=1}^{\infty} b_k[n] \leq \log_2(k)+1, \]
where we use the fact that there can be at most $\log_2(k)+1$ non-zero bits in the binary representation of $k$. This further implies that $\norm{f_k}_{p} \leq \norm{f_k}_1 \leq \log_2(k) +1 < \infty$ for all $p \in [1, \infty]$. Note that each $f_k$ maps a unit ball in $\Vcal$ to a subset of $\{\alpha\,  e_k \, : |\alpha| \leq \log_{2}(k) + 1\}$, which is a compact set for every $k \in \naturals$. Thus, for every $k \in \naturals$, $f_k$ is a compact operator and $f_k \in S_{1}(\Vcal, \Vcal) $. We trivially have $f_0 \in S_{1}(\Vcal, \Vcal)$.

\textbf{Proof of (i)}. Let $\sigma = \{\sigma_t\}_{t=1}^{T}$ be a sequence of i.i.d. Rademacher random variables. Consider a sequence of functions  $(x,y) = \{x_t, y_t\}_{t=1}^T$ such that $x_t(\sigma_{<t}) = e_t $ and $y_t(\sigma_{<t}) = 0$ for all $t \in [T]$. Note that our sequence $\{e_t\}_{t=1}^T \subseteq \Xcal$. Then, the sequential Rademacher complexity of the loss class is 
\begin{equation*}
    \begin{split}
        \text{Rad}_T(\Fcal) = \sup_{x, y} \,\expect \left[ \sup_{f \in \Fcal} \sum_{t=1}^T \sigma_t \norm{f(x_t(\sigma_{<t})) -y_t(\sigma_{<t})}^2\right] &\geq \expect \left[ \sup_{k \in \naturals} \sum_{t=1}^T \sigma_t \norm{f_k(e_t)}^2\right] \\
        &= \expect \left[ \sup_{k \in \naturals} \sum_{t=1}^T \sigma_t\, b_k[t] \right] \\
        &\geq \expect \left[ \sum_{t=1}^T \indicator\{\sigma_t =1\} \right]  = \frac{T}{2}.
\end{split}
\end{equation*}
Here, we use the fact that $f_k(e_t ) = b_k[t] \, e_k$ and $\prob[\sigma_t =1]= \frac{1}{2}$. As for the inequality $\sup_{k \in \naturals} \sum_{t=1}^T \sigma_t\, b_k[t]  \geq \sum_{t=1}^T \indicator\{\sigma_t =1\}$, note that for any sequence $\{\sigma_t\}_{t=1}^T$, there exists a  $k \in \naturals$ (possibly of the order $\sim 2^T$) such that $b_k[t]=1$ whenever $\sigma_t =1$ and $b_k[t] =0$ whenever $\sigma_t = -1$. 

\textbf{Proof of (ii)}. We now construct an online learner for $\Fcal$. Let $(x_1, y_1) \ldots, (x_T, y_T) \in \Xcal \times \Ycal$ denote the data stream. Since $y_t$ is an element of unit ball of $\Vcal$, we can write $y_t = \sum_{n=1}^{\infty}c_n(t) e_n$
such that $\sum_{n=1}^{\infty} c_n^2(t) \leq 1 $. For each $t \in [T]$, define a set of indices $S_t = \{n \in \naturals \, :\, |c_n(t)| \geq \frac{1}{2\sqrt{T}}\}$.  Since
\[1 \geq  \norm{y_t}^2 = \sum_{n=1}^{\infty} c_n^2(t) \geq \sum_{n \in S_t } c_n^2(t) \geq \sum_{n \in S_t} \frac{1}{4 T} = \frac{|S_t|}{4T},\]
we have $|S_t| \leq 4T$. Let $\text{sort}(S_i)$ denote the ordered list of size $4T$ that contains elements of $S_i$ in descending order.  If $S_i$ does not contain $4T$ indices, append $0$'s to the end of $\text{sort}(S_i)$. We let $\text{sort}(S_i)[j]$ denote the $j^{th}$ element of the ordered list $\text{sort}(S_i)$.

For each $i \in [T]$ and $j \in [4T]$, define an expert $E_i^j$ such that
\[E_i^j(x_t) = \begin{cases}
    0, \quad \quad  \quad \,\,t \leq i \\
    f_{k}(x_t), \quad t > i
\end{cases}, \quad \quad \text{ where } k = \text{sort}(S_i)[j]. \]
An online learner $\Acal$ for $\Fcal$ runs multiplicative weights algorithm using the set of experts $\mathcal{E} = \{E_i^j \, \mid i \in [T], j \in [4T]\}$. It is easy to see that $\norm{f_k(x)} \leq 1$ for all $x \in \Xcal$. Thus, for any $\hat{y}_t, y_t \in \Ycal$, we have $\norm{\hat{y}_t - y_t}^2 \leq 4$. Thus, for an appropriately chosen learning rate, the multiplicative weights algorithm guarantees (see Theorem 21.11 in \cite{ShwartzDavid}) that the regret of $\Acal$ satisfies
\[\expect\left[ \sum_{t=1}^T \norm{\Acal(x_t) -y_t}^2 \right] \leq \inf_{E \in \mathcal{E}} \sum_{t=1}^T \norm{E(x_t)-y_t}^2 +  4\sqrt{2T \ln(|\mathcal{E}|)}.\]
Note that $|\mathcal{E}| \leq 4T^2 $, which implies $4\sqrt{2T \ln(|\mathcal{E}|)} \leq 8 \sqrt{T \ln(2T)}$. We  now show that 
\begin{equation*}\label{sep:eq}
    \inf_{E \in \mathcal{E}} \sum_{t=1}^T \norm{E(x_t)-y_t}^2 \leq \inf_{f \in \Fcal} \sum_{t=1}^T \norm{f(x_t)-y_t}^2 + 2.
\end{equation*}
Together, these two inequalities imply that the expected regret of $\Acal$ is $\leq 2 + 8 \sqrt{T \ln(2T)}$. The rest of the proof is dedicated to proving the latter inequality. 

Let $f_{k^{\star}} \in \argmin_{f \in \Fcal} \sum_{t=1}^T \norm{f(x_t)-y_t}^2  $. Let $t^{\star} \in [T]$ be the first time point such that $k^{\star} \in S_{t^{\star}}$ and suppose it exists. Let $r^{\star} \in [4T]$ be such that $k^{\star} = \text{sort}(S_{t^{\star}})[r^{\star}]$. By definition of the experts, we have 
\[E_{t^{\star}}^{r^{\star}}(x_t) = f_{k^{\star}}(x_t) \quad \text{ for } t > t^{\star},\]
thus implying that $\sum_{t > t^{\star}} \norm{E_{t^{\star}}^{r^{\star}}(x_t)-y_t}^2 = \sum_{t > t^{\star}} \norm{f_{k^{\star}}(x_t)-y_t}^2$. 
Therefore, it suffices to show that
\begin{equation*}
    \sum_{t \leq t^{\star}} \norm{E_{t^{\star}}^{r^{\star}}(x_t)-y_t}^2 \leq \sum_{t \leq t^{\star}} \norm{f_{k^{\star}}(x_t)-y_t}^2 + 2.
\end{equation*}
As $E_{t^{\star}}^{r^{\star}}(x_{t}) = 0$ for all $t \leq t^{\star}$, proving the inequality above is equivalent to showing
\begin{equation*}
    \sum_{t \leq t^{\star}} \norm{y_t}^2 \leq \sum_{t \leq t^{\star}} \norm{f_{k^{\star}}(x_t)-y_t}^2 + 2.
\end{equation*}
Since $\norm{y_{t^{\star}}}^2 \leq 1$, we trivially have $\norm{ y_{t^{\star}}}^2 \leq \norm{f_{k^{\star}}(x_{t^{\star}})-y_{t^{\star}}}^2 + 1$.  Thus, by expanding the squared norm, the problem reduces to showing
\begin{equation*}
    \sum_{t < t^{\star}} \left( 2\inner{f_{k^{\star}}(x_t)}{y_t} - \norm{f_{k^{\star}}(x_t)}^2 \right) \leq  1.
\end{equation*}
We prove the inequality above by establishing 
\begin{equation*}
    2\inner{f_{k^{\star}}(x_t)}{y_t} - \norm{f_{k^{\star}}(x_t)}^2\leq \frac{1}{T} \quad \text{ for all } t < t^{\star}. 
\end{equation*}

Let $x_t = \sum_{n=1}^{\infty} \alpha_n(t) e_n$. We have $f_{k^{\star}}(x_t) = \sum_{n=1}^{\infty} b_{k^{\star}}[n] \inner{x_t}{e_n} e_{k^{\star}} = \left( \sum_{n=1}^{\infty} b_{k^{\star}}[n] \alpha_n(t)\right) e_{k^{\star}}$. Defining $a_{k^{\star}}(t) =  \left( \sum_{n=1}^{\infty} b_{k^{\star}}[n] \alpha_n(t)\right) $, we can write 
\[f_{k^{\star}}(x_t) = a_{k^{\star}}(t) e_{k^{\star}} \quad \text{ and } \quad  \norm{f_{k^{\star}}(x_t)} = |a_{k^{\star}}(t)|. \] 
So, it suffices to show that $2\,a_{k^{\star}}(t) \,c_{k^{\star}}(t) - |a_{k^{\star}}(t)|^2\leq \frac{1}{T}\text{ for all } t < t^{\star}.$
To prove this inequality, we consider the following two cases: 
\begin{itemize}
    \item[(I)] Suppose $|a_{k^{\star}}(t)| > 2 |c_{k^{\star}}(t)|$. Then, $2\,a_{k^{\star}}(t) \,c_{k^{\star}}(t) - |a_{k^{\star}}(t)|^2 < |a_{k^{\star}}(t)|^2 - |a_{k^{\star}}(t)|^2 = 0$.
    \item[(II)] Suppose $|a_{k^{\star}}(t)| \leq 2 |c_{k^{\star}}(t)|$. Then, $2\,a_{k^{\star}}(t) \,c_{k^{\star}}(t) - |a_{k^{\star}}(t)|^2 \leq 4\, |c_{k^{\star}}(t)|^2 < 4\, \left( \frac{1}{2\sqrt{T}}\right)^2 = \frac{1}{T}$ because $k^{\star} \notin S_t$ for all $t < t^{\star}$.
\end{itemize} 
In either case, $2\,a_{k^{\star}}(t) \,c_{k^{\star}}(t) - |a_{k^{\star}}(t)|^2\leq \frac{1}{T} \quad \text{for all } t < t^{\star}$. 

Finally, suppose that such a  $t^{\star}$ does not exist. Then, our analysis for the case $t \leq t^{\star}$ above shows that the expert $E_T^1$ that predicts $E_T^1(x_t) = 0$ for all $t \leq T$ satisfies
$
  \sum_{t=1}^T \norm{E_T^1(x_t)-y_t}^2 \leq  \sum_{t=1}^T \norm{f_{k^{\star}}(x_t)-y_t}^2 + 2.
$
 \end{proof} 

\subsection{ Batch Learnability without Uniform Convergence}
Although we state Theorem \ref{UCneqLearn}  in the online setting, an analogous result also holds in the batch setting. To establish the batch analog of Theorem \ref{UCneqLearn}, 
consider $f_k$ defined in \eqref{natfunc} and define a class $\Fcal = \{f_k \mid k \in \naturals\} \cup \{f_0\}$ where $f_0 =0$. This is the same class considered in the proof of Theorem \ref{UCneqLearn}.  Recall that in our proof of Theorem \ref{UCneqLearn} (i), we choose a sequence of labeled examples $\{e_t, 0\}_{t=1}^{T}$ that is independent of the sequence of Rademacher random variables $\{\sigma_t\}_{t=1}^T$. Thus, our proof shows that the i.i.d. version of the Rademacher complexity of $\Fcal$, where the labeled samples are independent of Rademacher variables, is also lower bounded by $\frac{T}{2}$. This implies that the class $\Fcal$ does not satisfy the uniform law of large numbers in the i.i.d. setting. However, using the standard online-to-batch conversion techniques, we can convert our online learner for $\Fcal$ to a batch learner for $\Fcal$ \citep{cesa2004generalization}. This shows a separation between uniform convergence and batch learnability of bounded linear operators. 

\section{Discussion and Open Questions}
 In this work, we study the online learnability of bounded linear operators between two infinite-dimensional Hilbert spaces.  In Theorems \ref{Sp} and \ref{Sp_lower}, we showed that 
 \[ c^2\, T^{1-\frac{1}{p}}\, \leq \inf_{\Acal} \,\texttt{R}_{\Acal}(T, 
 \Fcal_p) \leq 6c^2\, T^{\,\max \left\{\frac{1}{2}, 1-\frac{1}{p} \right\}},\]
 for every $p \in [1, \infty]$, where $\Fcal_{p} := \{f \in S_{p}(\Vcal, \Wcal) \, :\, \norm{f}_p \leq c\}$. Note that the upperbound and lowerbound match $p \geq 2$. However, for $p \in [1,2)$, the upperbound saturates at $\sqrt{T}$, while the lower bound gets progressively worse as $p$ decreases. Given this gap, we leave it open to resolve the following question.
 \begin{center}
     What is $\inf_{\Acal} \,\texttt{R}_{\Acal}(T, 
 \Fcal_p) $ for $p \in [1, 2)$? 
 \end{center}
\noindent We conjecture that lowerbound is loose  for $p \in [1,2)$, and one can obtain 
faster rates using some adaptation of the seminal Vovk-Azoury-Warmuth forecaster \citep{vovk2001competitive, azoury2001relative}.

Section \ref{sec:sep} shows a separation between sequential uniform convergence and online learnability for bounded linear operators. The separation is exhibited by a class that lies in $S_{1}(\Vcal, \Wcal)$, but is \textit{not} uniformly bounded. In this work, we established that there is no separation between online learnability and sequential uniform convergence for any subset of $S_p(\Vcal, \Wcal)$ with uniformly bounded $p$-Schatten norm for $p \in [1, \infty)$. However, it is unknown whether this is also true for $S_{\infty}(\Vcal, \Wcal)$. This raises the following natural question.
 \begin{center}
        Is $\text{Rad}_T(\Fcal) =o(T)$ if and only if $\inf_{\Acal} \texttt{R}_{\Acal}(T, \Fcal) =o(T)$ for every $\Fcal \subseteq \{f \in S_{\infty}(\Vcal, \Wcal) \mid \norm{f}_{\infty} \leq c\}$? 
 \end{center}

Finally, in this work, we showed that a uniform bound on the $p$-Schatten norm for any $p \in [1, \infty) $ is sufficient for online learnability. However, the example in Theorem \ref{UCneqLearn} shows that a uniform upper bound on the norm is not necessary for online learnability. Thus, it is an interesting future direction to fully characterize the landscape of learnability for bounded linear operators. In addition, it is also of interest to extend these results to nonlinear operators.

\acks{We acknowledge the assistance of Judy McDonald in locating a misdelivered package containing
\cite{ShwartzDavid}. Without the benefit of ideas in \cite{ShwartzDavid}, this paper would have never been written. AT acknowledges the support of NSF via grant IIS-2007055.  VR acknowledges the support of the NSF Graduate
Research Fellowship.}

\bibliography{references}

\appendix


\section{Upperbound Proofs for Online Setting}\label{appdx:tech_lem}

Our proof of Theorem \ref{Sp} also relies on the following technical Lemma.
\begin{lemma}\label{tr}
  Let $v \in \Vcal$, $w \in \Wcal$, and $f \in \Lcal(\Vcal, \Wcal)$. Then, we have $\inner{f(v)}{w} = \trace(f \circ (v \otimes w))$.
 \end{lemma}
\begin{proof}(of Lemma \ref{tr})
   Let $\{\psi_{n}\}_{n=1}^{\infty}$ be an orthonormal basis of $\Wcal$ and $w = \sum_{n=1}^{\infty} \alpha_n \psi_n$ for an $\ell_2$ summable sequence $\{\alpha_n\}_{n \in \naturals}$. Then, by definition of the trace operator, we have
\begin{equation*}
        \trace(f \circ (v \otimes w)) = \sum_{n=1}^{\infty} \inner{f \circ (v \otimes w)(\psi_n)}{\psi_n} = \sum_{n=1}^{\infty} \inner{\alpha_n \,f(v)}{\psi_n} = \inner{f(v)}{\sum_{n=1}^{\infty}\alpha_n \psi_n} = \inner{f(v)}{w},
\end{equation*}
which completes our proof. 
\end{proof}

\subsection{Proof of Lemma \ref{radsum}}
  Let $F = \sum_{t=1}^{T}\sigma_t\,  v_t(\sigma_{< t}) \otimes w_t(\sigma_{< t}) $. 
Since
\[\text{rank}\left(F\right) \leq \sum_{t=1}^T\text{rank}\left(\sigma_t\,  v_t(\sigma_{< t}) \otimes w_t(\sigma_{< t}) \right) \leq  T,  \] 
$F$ can have at most $T$ non-zero singular values. Let  $\{s_t\}_{t=1}^T$ be the singular values of the operator $F$, possibly with multiplicities. Then, for $q \in [1,2)$, we have
\begin{equation*}
    \begin{split}
        \norm{F}_q = \left(\sum_{t=1}^T s_t^{q} \right)^{\frac{1}{q}} \leq  \left(\left(\sum_{t=1}^T (s_t^q)^{\frac{2}{q}} \right)^{\frac{q}{2}}\, \left(\sum_{t=1}^T 1^{\frac{2}{2-q}} \right)^{\frac{2-q}{2}}\right)^{\frac{1}{q}} = \left( \sum_{t=1}^T s_t^2\right)^{\frac{1}{2}}\, T^{\frac{1}{q}-\frac{1}{2}} = \norm{F}_2 \,T^{\frac{1}{q}-\frac{1}{2}},
    \end{split}
\end{equation*}
where the inequality is due to H\"{o}lder. As for $q \geq 2$, we trivially have $\norm{F}_q \leq \norm{F}_2$. In either case, we obtain
\[\norm{F}_q \leq \max \left\{T^{\frac{1}{q}-\frac{1}{2}}, 1 \right\}\, \, \norm{F}_2. \]
Hence, to prove Lemma \ref{radsum}, it suffices to show that 
\[\expect[\, \norm{F}_2] \leq c_1\, c_2\,T^{\frac{1}{2}}.\]
Recall that by definition of the $2$-Schatten norm, we have $\norm{F}_2 = \sqrt{\trace\left( F^{\star}F\right)}$. Using linearity of trace and Jensen's inequality gives  $\expect \left[ \sqrt{\trace\left( F^{\star}F\right)}\right] \leq \sqrt{\trace\left( \expect\left[ 
 F^{\star}F\right]\right)}$. Then,
\begin{equation*}
    \begin{split}
    \expect\left[ 
 F^{\star}F\right] &= \expect \left[\left(\sum_{t=1}^{T}\sigma_t\,  w_t(\sigma_{< t}) \otimes v_t(\sigma_{< t}) \right)\left(\sum_{t=1}^{T}\sigma_t\,  v_t(\sigma_{< t}) \otimes w_t(\sigma_{< t}) \right) \right] \\
 &= \expect \left[\sum_{t, r}\sigma_t\,\sigma_r \inner{v_t(\sigma_{< t}) }{v_r(\sigma_{< r}) }  w_t(\sigma_{< t}) \otimes w_r(\sigma_{< r}) \right] \\
 &= \expect \left[\sum_{t=1}^T \norm{v_t(\sigma_{< t})}^2  w_t(\sigma_{< t}) \otimes w_t(\sigma_{< t}) \right] + \expect \left[\sum_{t \neq r} \sigma_t \sigma_r \inner{v_t(\sigma_{< t}) }{v_r(\sigma_{< r}) }  w_t(\sigma_{< t}) \otimes w_r(\sigma_{< r}) \right] \\
 &= \expect \left[\sum_{t=1}^T \norm{v_t(\sigma_{< t})}^2  w_t(\sigma_{< t}) \otimes w_t(\sigma_{< t}) \right].
    \end{split}
\end{equation*}
To see why the second term above is $0$, consider the case $t < r$. We have 
\begin{equation*}
    \begin{split}
       \expect \left[\sigma_t \sigma_r \inner{v_t(\sigma_{< t}) }{v_r(\sigma_{< r}) }  w_t(\sigma_{< t}) \otimes w_r(\sigma_{< r}) \right] &= \expect\left[\expect \left[\sigma_t \sigma_r \inner{v_t(\sigma_{< t}) }{v_r(\sigma_{< r}) }  w_t(\sigma_{< t}) \otimes w_r(\sigma_{< r}) \mid \sigma_{< r}\right] \right] \\
       &= \expect\left[\sigma_t \inner{v_t(\sigma_{< t}) }{v_r(\sigma_{< r}) }  w_t(\sigma_{< t}) \otimes w_r(\sigma_{< r})\, \expect \left[ \sigma_r\mid \sigma_{< r}\right] \right] \\
       &=0.
    \end{split}
\end{equation*}
The last equality follows because $\sigma_r$ is independent of $\sigma_{<r}$ and thus $\expect \left[ \sigma_r\mid \sigma_{< r}\right]  = \expect[\sigma_r] =0 $. The case where $t > r$ is symmetric. Putting everything together, we have
\begin{equation*}
    \begin{split}
       \trace\left( \expect[F^{\star}F]\right) &= \trace \left(\expect \left[\sum_{t=1}^T \norm{v_t(\sigma_{< t})}^2  w_t(\sigma_{< t}) \otimes w_t(\sigma_{< t}) \right]  \right)\\
       &=\expect \left[\sum_{t=1}^T \norm{v_t(\sigma_{< t})}^2  \trace \left(w_t(\sigma_{< t}) \otimes w_t(\sigma_{< t})\right) \right]  \\
       &= \expect \left[\sum_{t=1}^T \norm{v_t(\sigma_{< t})}^2 \norm{w_t(\sigma_{< t})}^2  \right] \\
       &\leq \sum_{t=1}^T c_1^2c_2^2 = (c_1c_2)^2\,T,
    \end{split}
\end{equation*}
which implies that $\expect[\, \norm{F}_2]\leq    \sqrt{\trace\left(\expect \left[ F^{\star}F\right]\right)} \leq \sqrt{(c_1c_2)^2 T} = c_1\, c_2\,T^{\frac{1}{2}}$. This completes our proof.

\subsection{Proof of Theorem \ref{Sp}}

    Define the normalized loss class $ \{(u, v) \mapsto \frac{1}{4c^2}\norm{f(u)-v}^2 \, : f \in \Fcal_p\}$ such that every function in this class maps to $[0,1]$. Applying
  \cite[Theorem 2]{rakhlin2015sequential} to this normalized loss class, we obtain that the expected regret of $\Acal$ is $\leq 8c^2 \, \text{Rad}_T(\overline{\Fcal}_p)$, where $\overline{\Fcal}_p = \{\frac{1}{4c^2} f \mid f \in \Fcal_p\}$ is the normalized operator class. Since $\text{Rad}_T(\overline{\Fcal}_p) = \frac{1}{4c^2}\, \text{Rad}_T(\Fcal_p)$, the expected regret of $\Acal $ is $\leq 2\, \text{Rad}_T(\Fcal_p)$. This completes the proof of the first inequality. We now focus on proving the second inequality here. By definition, we have
  \begin{equation*}
    \begin{split}
            &\text{Rad}_{T}( \Fcal_p) = \sup_{x,y} \,\expect\left[ \sup_{f \in \Fcal_p} \,\sum_{t=1}^T \sigma_t \norm{f(x_t(\sigma_{< t}))- y_t(\sigma_{<t})}^2\right] \\
            &\leq \sup_{x,y} \left(\,\expect\left[ \sup_{f \in \Fcal_p} \,\sum_{t=1}^T \sigma_t \norm{f(x_t(\sigma_{< t}))}^2\right] + 2\expect\left[ \sup_{f \in \Fcal_p} \,\sum_{t=1}^T -\sigma_t \inner{f(x_t(\sigma_{< t}))}{y_t(\sigma_{<t})}\right] \right. \\
            &\quad \left. + \expect\left[ \,\sum_{t=1}^T \sigma_t \norm{y_t(\sigma_{< t}))}^2\right]\right)\\
            &= \sup_{x,y} \left(\,\expect\left[ \sup_{f \in \Fcal_p} \,\sum_{t=1}^T \sigma_t \norm{f(x_t(\sigma_{< t}))}^2\right] + 2\expect\left[ \sup_{f \in \Fcal_p} \,\sum_{t=1}^T \sigma_t \inner{f(x_t(\sigma_{< t}))}{y_t(\sigma_{<t})}\right] \right).
    \end{split}
    \end{equation*}

To handle the second term above, recall that Lemma \ref{tr} implies $\inner{f(x_t(\sigma_{< t}))}{y_t(\sigma_{<t})} = \trace(f \circ \, (x_t(\sigma_{< t}) \otimes y_t(\sigma_{< t})))$. Using the linearity of the trace operator, we obtain 
\begin{equation*}
    \begin{split}
        \sum_{t=1}^T \sigma_t \inner{f(x_t(\sigma_{< t}))}{y_t(\sigma_{<t})}  = \trace\left( f \circ \sum_{t=1}^{T}\sigma_t\,  x_t(\sigma_{< t}) \otimes y_t(\sigma_{< t})\right) \leq \norm{f}_p \, \, \norm{\sum_{t=1}^{T}\sigma_t\,  x_t(\sigma_{< t}) \otimes y_t(\sigma_{< t})}_q,
    \end{split}
\end{equation*}
where $q := 1-\frac{1}{p}$ is the H\"{o}lder conjugate of $p$\, \cite[Page 41]{reed1975ii}.
This implies the bound
\begin{equation*}
    \begin{split}
      \expect\left[ \sup_{f \in \Fcal_p} \,\sum_{t=1}^T \sigma_t \inner{f(x_t(\sigma_{< t}))}{y_t(\sigma_{<t})}\right] &\leq \expect \left[\sup_{f \in \Fcal_p}\, \norm{f}_p \, \, \norm{\sum_{t=1}^{T}\sigma_t\,  x_t(\sigma_{< t}) \otimes y_t(\sigma_{< t})}_q \right]  \\
      &\leq c\,  \expect \left[\norm{\sum_{t=1}^{T}\sigma_t\,  x_t(\sigma_{< t}) \otimes y_t(\sigma_{< t})}_q \right], 
    \end{split}
\end{equation*}
where the last inequality follows from the definition of $\Fcal_p$.

To handle the first term in the bound of $\text{Rad}_T(\Fcal_p)$ above,  note that
 \[\norm{f(x_t(\sigma_{< t}))}^2 = \inner{f(x_t(\sigma_{< t}))}{f(x_t(\sigma_{< t}))} = \inner{f^{\star}f(x_t(\sigma_{< t}))}{x_t(\sigma_{< t})} =\trace(f^{\star}f \circ (x_t(\sigma_{< t}) \otimes x_t(\sigma_{< t}))),\]
 where the final equality follows from Lemma \ref{tr}. Using linearity of trace, and the generalized H\"{o}lder's inequality for Schatten norms \cite[Page 41]{reed1975ii}, we obtain
\begin{equation*}
    \begin{split}
       \expect\left[ \sup_{f \in \Fcal_p} \,\sum_{t=1}^T \sigma_t \norm{f(x_t(\sigma_{< t}))}^2\right] &\leq \expect \left[ \sup_{f \in \Fcal_p} \, \norm{f^{\star}f}_p \norm{\sum_{t=1}^{T}\sigma_t\,  x_t(\sigma_{< t}) \otimes x_t(\sigma_{< t})}_q \right] \\
       &\leq c^2  \expect \left[\norm{\sum_{t=1}^{T}\sigma_t\,  x_t(\sigma_{< t}) \otimes x_t(\sigma_{< t})}_q \right], 
    \end{split}
\end{equation*}
where the last inequality uses the fact that $\norm{f^{\star}f}_p \leq \norm{f}_p^2$.
Combining everything, we obtain
\begin{equation*}
\begin{split}
        \text{Rad}_{T}( \Fcal_p) \leq c^2 \expect \left[\norm{\sum_{t=1}^{T}\sigma_t\,  x_t(\sigma_{< t}) \otimes x_t(\sigma_{< t})}_q \right]+ 2c\, \expect \left[\norm{\sum_{t=1}^{T}\sigma_t\,  x_t(\sigma_{< t}) \otimes y_t(\sigma_{< t})}_q \right] \leq  3c^2 \, T^{\max\left\{\frac{1}{2}, \frac{1}{q}\right\}},
\end{split}
\end{equation*}
where the final inequality follows from using Lemma \ref{radsum} twice. Recalling that $\frac{1}{q} = 1- \frac{1}{p}$ completes our proof of second inequality.

\section{Proof of Theorem \ref{Sp_lower_PAC}}\label{appdx:PAC}
\subsection{Proof of lowerbound of $\frac{c^2}{12}\, n^{-\frac{1}{p-1}}$.}\label{apppdx:PAC1}
\begin{proof}
     Fix $n, m \in \naturals$. Let $\Dcal$ be an arbitrary joint distribution on $\Xcal \times \Ycal$, and $U$ denote the uniform distribution on $\{e_1, \ldots, e_{mn}\}$. For each $\sigma \in \{-1,1\}^{mn}$, define 
$h_{\sigma} =\sum_{i=1}^{mn} c\,\sigma_i\,  \psi_i \otimes e_i$. Note that $h_{\sigma} \notin \Fcal_p$ for large $n$. The minimax expected excess risk of $\Fcal$ is
\begin{equation*}
    \begin{split}
     \Ecal_{n}(\Fcal) &= \inf_{\hat{f}_n} \sup_{\Dcal }  \expect_{S \sim \Dcal^n}\left[ \expect_{(x,y)\sim \Dcal} \left[ \norm{\hat{f}_n(x)-y}^2\right] - \inf_{f \in \Fcal_p} \expect_{(x,y)\sim \Dcal} \left[ \norm{f(x)-y}^2\right]\right] \\
    &\geq \inf_{\hat{f}_n} \expect_{\sigma \sim \{\pm 1\}^{mn}} \left[ \expect_{S \sim (U \times h_{\sigma})^n}\left[ \expect_{x\sim U} \left[ \norm{\hat{f}_n(x)-h_{\sigma}(x)}^2\right] - \inf_{f \in \Fcal_p} \expect_{x\sim U} \left[ \norm{f(x)-h_{\sigma}(x)}^2\right]\right]\right],
    \end{split}
\end{equation*}
where the first inequality follows upon replacing supremum over $\Dcal, \sigma$ with $U$ and expectation over $\sigma$ respectively.  Let $S_x \in \Xcal^n$ denote the instances from labeled samples $S \in (\Xcal \times \Ycal)^n$. 
We first lower bound the expected risk of the learner, and then upper bound the expected risk of the optimal function in $\Fcal_p$.
Exchanging the order of the first two expectations, the lower bound of the expected risk of the learner is
\begin{equation*}
    \begin{split}
  &\inf_{\hat{f}_n} \expect_{S_x \sim U^n}\left[\expect_{\sigma \sim \{\pm 1\}^{mn}}  \left[\expect_{x \sim U} \left[ \norm{\hat{f}_n(x)-h_{\sigma}(x)}^2\right] \right] \right]\\
    &=\inf_{\hat{f}_n} \expect_{S_x \sim U^n}\left[\expect_{\sigma \sim \{\pm 1\}^{mn}}  \left[ \frac{1}{mn} \sum_{i=1}^{mn}\norm{\hat{f}_n(e_i)-h_{\sigma}(e_i)}^2\right]  \right]\\
    &\geq \inf_{\hat{f}_n} \expect_{N \sim \text{Unif}(\{1,, \ldots, mn\})^{n}}\left[\expect_{\sigma \sim \{\pm 1\}^{mn}}  \left[ \frac{1}{mn} \sum_{i \notin N}\norm{\hat{f}_n(e_i)-c\,\sigma_i \psi_i}^2\right]  \right]\\
    &\geq \inf_{\hat{f}_n} \expect_{N \sim \text{Unif}(\{1,, \ldots, mn\})^{n}}\left[ \frac{1}{mn} \sum_{i \notin N}\left(\expect_{\sigma \sim \{\pm 1\}^{mn}}\left[ \norm{\hat{f}_n(e_i)-c\,\sigma_i \psi_i}\right] \right)^2 \right].
    \end{split}
\end{equation*}
In order to get the second to the last inequality, we reinterpret sampling $x$ uniformly from $\{e_1, \ldots, e_{mn}\}$ as sampling index $i$ uniformly from $\{1, \ldots, mn\}$ and drawing $e_i$. The final inequality follows upon exchanging the sum and expectation and applying Jensen's. Note that, whenever $i \notin N$, we have 
\begin{equation*}
    \begin{split}
      &\expect_{\sigma \sim \{\pm 1\}^{mn}}\left[ \norm{\hat{f}_n(e_i)-c\,\sigma_i \psi_i}\right] = \expect \left[ \expect_{\sigma_i} \left[\norm{\hat{f}_n(e_i)-c\,\sigma_i \psi_i} \right] \mid \hat{f}_n \right]  \\
      &= \expect \left[ \frac{1}{2} \left( \norm{\hat{f}_n(e_i)-c\,\psi_i} + \norm{\hat{f}_n(e_i)+c\,\psi_i}\right) \mid \hat{f}_n \right] \\
      &\geq \frac{1}{2} \norm{c\psi_i + c\psi_i}\\
      &= c,
    \end{split}
\end{equation*}
where we use the fact $\hat{f}_n$ is independent of $\sigma_i$ for all $i \notin N$ and triangle inequality. Thus, combining everything, our lower bound is
\[\geq \inf_{\hat{f}_n} \expect_{N \sim \text{Unif}(\{1,, \ldots, mn\})^{n}}\left[ \frac{1}{mn} \sum_{i \notin N}c^2  \right]  = \frac{c^2}{mn} \sum_{i=1}^{mn} \mathbb{P}(i \notin N) = c^2 \left(1-\frac{1}{mn} \right)^n. \] 
For the last equality, we use the fact that the probability of $i$ not appearing in the set $N$ obtained by $n$ random uniform draw from $\{1,2, \ldots, mn\}$ with replacement is $\left(1-\frac{1}{mn} \right)^n$.

Next, we upperbound optimal expected risk amongst functions in $\Fcal_p$. Consider \[f_{\sigma,p} = 
 \sum_{j=1}^{mn}  \frac{c\, \sigma_j \,}{(mn)^{1/p}}\,  \psi_j \otimes e_j.\] 
 Clearly, $\norm{f_{\sigma,p}}_p \leq c$ for all $p \in [1, \infty]$ and thus $f_{\sigma,p} \in \Fcal_p$. Therefore, we can write
\begin{equation*}
    \begin{split}
        \inf_{f \in \Fcal_p} \expect_{x\sim U} \left[ \norm{f(x)-h_{\sigma}(x)}^2\right]  &\leq \expect_{x\sim U} \left[ \norm{f_{\sigma,p}(x)-h_{\sigma}(x)}^2\right] \\
        &= \frac{1}{mn}\sum_{i=1}^{mn} 
        \norm{f_{\sigma,p}(e_i) -h_{\sigma}(e_i) }^2\\
        &= \frac{1}{mn}\sum_{i=1}^{mn} \norm{\frac{c \, \sigma_i}{(mn)^{1/p}}\,  \psi_i - c\,\sigma_i \psi_i}^2\\
        &= \frac{1}{mn} \sum_{i=1}^{mn}c^2\, \left(1- \frac{1}{(mn)^{1/p}} \right)^2\\
       &= c^2 \left( 1- \frac{1}{(mn)^{1/p}} \right)^2 \leq c^2 \left( 1-\frac{1}{(mn)^{1/p}}\right).    
       \end{split}
\end{equation*}
Thus, putting everything together, the minimax expected excess risk is
\begin{equation*}
    \begin{split}
       &\geq c^2 \left(1-\frac{1}{mn} \right)^n - c^2 \left( 1-\frac{1}{(mn)^{1/p}}\right) \\
       &\geq c^2 \left(1-\frac{1}{2m} \right)^2 - c^2 \left( 1-\frac{1}{(mn)^{1/p}}\right) \quad \quad (\text{ for } n\geq 2)\\
       &\geq c^2 \left(\frac{1}{(mn)^{\frac{1}{p}}}-   \frac{1}{2m}\right). \\
    \end{split}
\end{equation*}
Next, pick $m =  \,\ceil{2n^{\frac{1}{p-1}}}$. Then, we have that $2n^{\frac{1}{p-1}} \leq m \leq 3n^{\frac{1}{p-1}}$. So, the expression above is further lower bounded by
\[c^2 \left(\frac{1}{(3n^{\frac{1}{p-1}}\, n)^{\frac{1}{p}}}-   \frac{1}{2\, 2n^{\frac{1}{p-1}}}\right) \geq c^2 \left( \frac{1}{3n^{\frac{1}{p-1}}} - \frac{1}{4n^{\frac{1}{p-1}}}\right) = \frac{c^2}{12 n^{\frac{1}{p-1}}}.\]
This completes our proof. 
\end{proof}

\subsection{Proof of lowerbound of $\frac{c^2}{8}\, n^{-\frac{2}{p}}$.}
Our proof here follows similar arguments as the proof in \ref{apppdx:PAC1}. However, the lowerbound in this section is derived in the realizable setting. 
\begin{proof}
     Fix $n, m \in \naturals$. Let $\Dcal$ be an arbitrary joint distribution on $\Xcal \times \Ycal$, and let $U$ denote the uniform distribution on $\{e_1, \ldots, e_{mn}\}$. For each $\sigma \in \{-1,1\}^{mn}$, define 
$f_{\sigma,p} =\sum_{i=1}^{mn} \frac{c}{(mn)^{1/p}}\,\sigma_i\,  \psi_i \otimes e_i$. Note that $f_{\sigma,p} \in \Fcal_p$ for all $p \geq 1$. The minimax expected excess risk of $\Fcal$ is
\begin{equation*}
    \begin{split}
     \Ecal_n(\Fcal)&=\inf_{\hat{f}_n} \sup_{\Dcal }  \expect_{S \sim \Dcal^n}\left[ \expect_{(x,y)\sim \Dcal} \left[ \norm{\hat{f}_n(x)-y}^2\right] - \inf_{f \in \Fcal_p} \expect_{(x,y)\sim \Dcal} \left[ \norm{f(x)-y}^2\right]\right] \\
    &\geq \inf_{\hat{f}_n} \expect_{\sigma \sim \{\pm 1\}^{mn}} \left[ \expect_{S \sim (U \times f_{\sigma,p})^n}\left[ \expect_{x\sim U} \left[ \norm{\hat{f}_n(x)-f_{\sigma,p}(x)}^2\right] - \inf_{f \in \Fcal_p} \expect_{x\sim U} \left[ \norm{f(x)-f_{\sigma,p}(x)}^2\right]\right]\right] \\
    &\geq \inf_{\hat{f}_n} \expect_{\sigma \sim \{\pm 1\}^{mn}} \left[ \expect_{S \sim (U \times f_{\sigma,p})^n}\left[ \expect_{x\sim U} \left[ \norm{\hat{f}_n(x)-f_{\sigma,p}(x)}^2\right] \right]\right]
    \end{split}
\end{equation*}
where the first inequality follows upon replacing supremum over $\Dcal, \sigma$ with $U$ and expectation over $\sigma$. The second inequality follows because $\inf_{f \in \Fcal_p} \expect_{x\sim U} \left[ \norm{f(x)-f_{\sigma,p}(x)}^2\right] \leq \expect_{x\sim U} \left[ \norm{f_{\sigma,p}(x)-f_{\sigma,p}(x)}^2\right] =0$ as $f_{\sigma,p} \in \Fcal_p$.

Let $S_x$ denote the instances from labeled samples $S$. 
We first lower bound the expected risk of the learner $\hat{f}_n$. Following the same calculation as in the first part of the proof, the lower bound of the expected risk of the learner is
\begin{equation*}
    \begin{split}
  &\inf_{\hat{f}_n} \expect_{S_x \sim U^n}\left[\expect_{\sigma \sim \{\pm 1\}^{mn}}  \left[\expect_{x \sim U} \left[ \norm{\hat{f}_n(x)-f_{\sigma,p}(x)}^2\right] \right] \right]\\
    &=\inf_{\hat{f}_n} \expect_{S_x \sim U^n}\left[\expect_{\sigma \sim \{\pm 1\}^{mn}}  \left[ \frac{1}{mn} \sum_{i=1}^{mn}\norm{\hat{f}_n(e_i)-f_{\sigma,p}(e_i)}^2\right]  \right]\\
    &\geq \inf_{\hat{f}_n} \expect_{N \sim \text{Unif}(\{1,, \ldots, mn\})^{n}}\left[\expect_{\sigma \sim \{\pm 1\}^{mn}}  \left[ \frac{1}{mn} \sum_{i \notin N}\norm{\hat{f}_n(e_i)-\,\frac{c\,\sigma_i }{(mn)^{1/p}}\psi_i}^2\right]  \right]\\
    &\geq \inf_{\hat{f}_n} \expect_{N \sim \text{Unif}(\{1,, \ldots, mn\})^{n}}\left[ \frac{1}{mn} \sum_{i \notin N}\left(\expect_{\sigma \sim \{\pm 1\}^{mn}}\left[ \norm{\hat{f}_n(e_i)-\frac{c\, \sigma_i }{(mn)^{1/p}}\psi_i }\right] \right)^2 \right].
    \end{split}
\end{equation*}
To get the second to the last inequality, we reinterpret sampling $x$ uniformly from $\{e_1, \ldots, e_{mn}\}$ as sampling index $i$ uniformly from $\{1, \ldots, mn\}$ and drawing $e_i$. The final inequality follows upon exchanging the sum and expectation and applying Jensen's. Note that, whenever $i \notin N$, we have 
\begin{equation*}
    \begin{split}
      &\expect_{\sigma \sim \{\pm 1\}^{mn}}\left[ \norm{\hat{f}_n(e_i)-c\,\sigma_i \psi_i}\right] = \expect \left[ \expect_{\sigma_i} \left[\norm{\hat{f}_n(e_i)-\frac{c\,\sigma_i}{(mn)^{1/p}} \psi_i} \right] \mid \hat{f}_n \right]  \\
      &= \expect \left[ \frac{1}{2} \left( \norm{\hat{f}_n(e_i)- \frac{c}{(mn)^{1/p}}\,\psi_i} + \norm{\hat{f}_n(e_i)+\frac{c}{(mn)^{1/p}}}\right) \mid \hat{f}_n \right] \\
      &\geq \frac{c}{(mn)^{1/p}}
    \end{split}
\end{equation*}
where we use the fact $\hat{f}_n$ is independent of $\sigma_i$ as $i \notin N$ and triangle inequality. Thus, combining everything, our lower bound is
\[\geq \inf_{\hat{f}_n} \expect_{N \sim \text{Unif}(\{1,, \ldots, mn\})^{n}}\left[ \frac{1}{mn} \sum_{i \notin N}\frac{c^2}{(mn)^{2/p}} \right]  =  \frac{c^2}{(mn)^{2/p}} \left(1-\frac{1}{mn} \right)^n. \] 
For the last equality, we use the fact that the probability of $i$ not appearing in the set $N$ obtained by $n$ random uniform draw from $\{1,2, \ldots, mn\}$ with replacement is $\left(1-\frac{1}{mn} \right)^n$. Picking $m=2$ and using the fact that $\left(1-\frac{1}{2n} \right)^n \geq 1-1/2 = 1/2$, we obtain the lowerbound of $\frac{c^2}{8} n^{-\frac{2}{p}}$.

\end{proof}

\end{document}